\PassOptionsToPackage{table}{xcolor}
\documentclass{article}

\usepackage{iclr2027_conference,times}
\usepackage[T1]{fontenc}
\usepackage{amsmath,amsfonts,bm}

\def\eqref#1{equation~\ref{#1}}

\def\1{\bm{1}}

\DeclareMathAlphabet{\mathsfit}{\encodingdefault}{\sfdefault}{m}{sl}
\SetMathAlphabet{\mathsfit}{bold}{\encodingdefault}{\sfdefault}{bx}{n}

\usepackage{url}
\usepackage{graphicx}
\usepackage{amsmath}
\usepackage{amssymb}
\usepackage{amsthm}
\usepackage{booktabs}
\usepackage{makecell}
\usepackage{multirow}
\usepackage{array}
\usepackage{tabularx}
\usepackage{placeins}
\usepackage{float}
\usepackage{wrapfig}
\usepackage{needspace}
\usepackage{capt-of}
\usepackage{xcolor}
\usepackage[normalem]{ulem}
\usepackage{algorithm}
\usepackage{algpseudocode}
\usepackage{hyperref}
\hypersetup{hidelinks}

\newtheoremstyle{carmplain}{6pt}{6pt}{\itshape}{}{\bfseries}{.}{.5em}
  {\thmname{#1}\thmnumber{ #2}\thmnote{ \textbf{(#3)}}}
\theoremstyle{carmplain}
\newtheorem{proposition}{Proposition}
\newtheorem{corollary}{Corollary}

\newcommand{\method}{CARM}
\newcommand{\methodours}{\method{} (ours)}
\newcommand{\ind}{\mathbf{1}}
\newcommand{\bestresult}[1]{\ensuremath{\mathbf{#1}}}
\newcommand{\secondbest}[1]{\underline{#1}}
\newcommand{\relatedheading}[1]{%
  \noindent{\fontencoding{T1}\fontfamily{ptm}\fontseries{b}\selectfont #1}}

\title{CARM: Cancellation-Aware Response\\
Masking for LLM Reinforcement Learning}

\author{Yafei Zhang \quad Songshuo Lu \quad Sicong Liao \quad Zhi Chen \quad Yaohua Tang\\[0.4em]
\normalfont Moore Threads AI}

\iclrfinalcopy

\date{}
\hypersetup{pdftitle={CARM: Cancellation-Aware Response Masking for LLM Reinforcement Learning},pdfauthor={Yafei Zhang, Songshuo Lu, Sicong Liao, Zhi Chen, Yaohua Tang}}
\fancypagestyle{carmfirst}{%
  \fancyhf{}%
  \fancyhead[L]{\includegraphics[height=0.65cm]{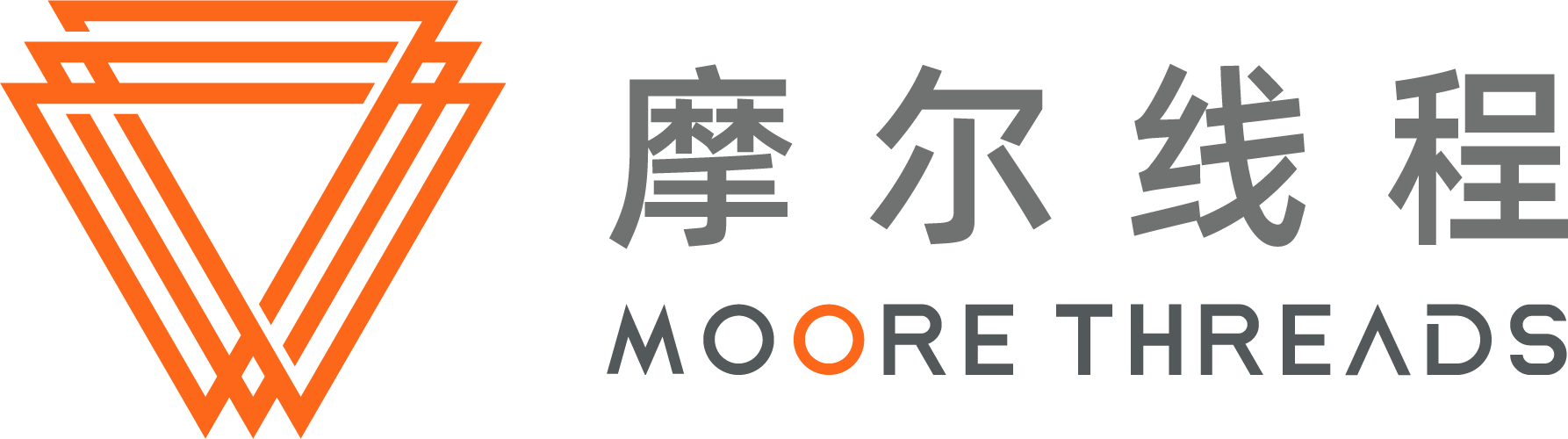}}%
  \fancyfoot[C]{\thepage}%
  \renewcommand{\headrulewidth}{0.8pt}%
  \renewcommand{\footrulewidth}{0pt}%
}

\begin{document}
\maketitle
\fancyhf{}
\fancyfoot[C]{\thepage}
\renewcommand{\headrulewidth}{0pt}
\renewcommand{\footrulewidth}{0pt}
\thispagestyle{carmfirst}

\begin{abstract}
Recent years have witnessed the rapid adoption of reinforcement learning (RL)
in large language model (LLM) post-training, with substantial gains in
mathematical reasoning and code generation. In practical systems, however,
policy updates and differences between rollout and training engines can make
sampled responses off-policy. Sequence-level masking addresses this
mismatch by deciding whether an entire response should contribute to
optimization. A common masking rule uses the length-normalized geometric mean
of sampled token probability ratios. Its signed log-ratios can cancel across
positions, concealing substantial bidirectional policy drift. We propose
\emph{Cancellation-Aware Response Masking} (\method), a sequence-level mask
that takes the absolute value of each token log-ratio before averaging,
preventing opposing probability changes from canceling. We prove that
accepted responses satisfy a joint bound on the fraction of sampled-token
ratios outside a prescribed band and their mean log-distance beyond its
boundaries. Experiments on mathematical reasoning and code generation
show that \method{} improves mean@16
averaged over AIME 2024/2025/2026 and BeyondAIME by up to $3.13$ percentage
points over geometric-mean masking, and increases average pass@1 across four
code benchmarks by $2.88$ points over the strongest evaluated baseline.
These findings support \method{} as a theoretically grounded and effective
method for response-level off-policy control in LLM reinforcement learning.
\end{abstract}

\section{Introduction}
\label{sec:introduction}

Reinforcement learning (RL) is widely used to improve the reasoning ability of
large language models (LLMs), building on advances in preference-based
alignment and execution-feedback learning \citep{ouyang2022instructgpt,le2022coderl}.
In an ideal on-policy setting, responses are generated by the same policy
that is being optimized. In practical systems, however, several effects
introduce policy mismatch. Mini-batch optimization makes stored responses
stale as the policy changes, while actor--learner latency can delay their
use. High-throughput rollout engines may also differ from the training
backend in batching, numerical precision, and kernel implementations
\citep{yao2025offpolicy,sheng2025hybridflow}. Consequently, the policy being
updated can assign substantially different probabilities to sampled tokens
than the rollout policy that generated them.

Proximal Policy Optimization (PPO) and Group Relative Policy Optimization
(GRPO) regulate policy updates through clipping in their token-level
surrogate losses
\citep{schulman2017ppo,shao2024deepseekmath,guo2025deepseek}.
Clipping changes how individual tokens contribute to the loss, but does not
keep every current-to-rollout probability ratio within a prescribed interval.
It also does not decide whether an entire response should participate in
the update. Sequence-level masking complements clipping by making this
decision from a score that summarizes mismatch across the response.
Trust Region Masking (TRM), for example, motivates such filtering through
an analysis of how policy mismatch affects optimization over long
autoregressive sequences \citep{li2026trust}.

A simple response-level criterion is geometric-mean masking: it takes the
geometric mean of the token probability ratios and retains responses when
this aggregate is close to one. We refer to this baseline as \emph{GeoMean}.
In log space, this aggregation corresponds to averaging signed token
log-ratios, as used in DeepSeek-V3.2's off-policy sequence mask
\citep{liu2025deepseek}. The problem is that positive and negative
log-ratios can cancel. For example, token ratios of $10$ and $0.1$ have a
geometric mean of one, although each sampled-token probability has changed
tenfold. More generally, a response can appear well aligned under this
score because its probability changes balance out, rather than because
those changes are small.

We propose \emph{Cancellation-Aware Response Masking} (\method), a
sequence-level mask that takes the absolute value of each token log-ratio
before averaging. This prevents opposing probability changes from canceling
while treating reciprocal changes symmetrically.
Figure~\ref{fig:null_empirical_case} illustrates the distinction on a real
8,619-token response: substantial positive and negative log-ratio
contributions largely cancel under GeoMean, whereas \method{} retains their
magnitudes. \method{} changes which responses contribute to optimization
while retaining the token-level PPO/GRPO surrogate.

Beyond preventing cancellation, \method{} gives the response-level
threshold a direct token-level interpretation. We prove that every accepted
response satisfies a joint bound on the fraction of sampled-token ratios
outside a prescribed, possibly asymmetric band and their mean log-distance
beyond its boundaries. Experiments on mathematical reasoning and code
generation show improvements over token- and sequence-level baselines.
Threshold analyses further show that similar filtering rates can yield
different accuracies, while filtering more responses is not consistently
better.

\begin{figure}[t]
    \centering
    \includegraphics[width=\linewidth]{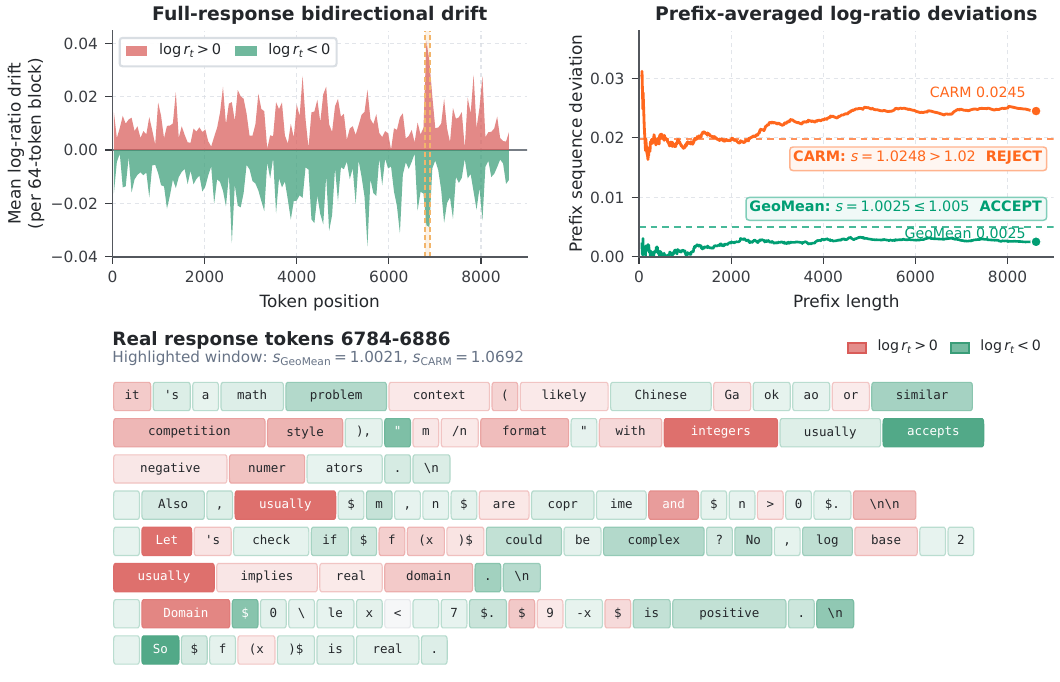}
    \caption{A real unfiltered rollout that GeoMean accepts and \method{}
    rejects. This 8,619-token response has scores
    $s_{\mathrm{GeoMean}}=1.0025\le1.005$ and
    $s_{\mathrm{CARM}}=1.0248>1.02$. \emph{Top:} positive and negative
    log-ratio mass and prefix deviations. \emph{Bottom:} a 103-token window;
    red tokens have $\log r_t>0$, green tokens have $\log r_t<0$, and color
    intensity encodes magnitude.}
    \label{fig:null_empirical_case}
\end{figure}

Our contributions are:
\begin{itemize}
    \item We formalize cancellation in signed-log sequence masking: the
    aggregate score can indicate perfect apparent alignment even when every
    sampled token ratio violates a prescribed band.
    \item We propose \method, a sequence-level mask that prevents cancellation
    by taking absolute values before averaging token log-ratios. We derive a
    joint bound on the frequency and severity of ratio-band violations in
    accepted responses, including asymmetric bands, giving the filtering
    threshold a token-level interpretation.
    \item Experiments on mathematical reasoning and code generation show
    improvements over sequence- and token-level baselines. Paired-response
    diagnostics and threshold analyses characterize the selection behavior:
    higher accuracy is observed at comparable average filtering rates, and
    accuracy is not monotone in the masked fraction.
\end{itemize}

\section{Related Work}
\label{sec:related_work}

\relatedheading{Policy optimization for LLM post-training.}
Trust Region Policy Optimization (TRPO) constrains Kullback--Leibler (KL)
divergence, while PPO uses a clipped-ratio surrogate
\citep{schulman2015trpo,schulman2017ppo}. GRPO replaces a learned critic with
within-group relative advantages \citep{shao2024deepseekmath}.
Answer correctness and program execution provide reward signals for improving
mathematical reasoning and code generation
\citep{le2022coderl,dou2024stepcoder,guo2025deepseek,yang2025qwen3,glm2026glm5},
complementing inference-time reasoning and program-aided prompting
\citep{wei2022cot,chen2023pot,gao2023pal}.
Decoupled Clip and Dynamic sAmpling Policy Optimization (DAPO) refines
clipping and sampling \citep{yu2025dapo}; related work revisits advantage
normalization, sequence-level objectives, and reward--policy learning
\citep{hu2025reinforcepp,xie2025spo,lu2026urpo}.
These choices determine how sampled responses contribute to optimization.
\method{} addresses the complementary question of whether a sampled response
should contribute to the current policy update.

\relatedheading{Off-policy and token-level trust-region control.}
This selection problem arises even in nominally online training: actor--learner
delay, repeated updates, and rollout--training engine differences separate
the policy that generates responses from the policy being optimized
\citep{espeholt2018impala,yao2025offpolicy,sheng2025hybridflow}.
Importance weighting and token-level filtering control individual update
contributions \citep{espeholt2018impala,ling2025ring1t}. Adaptive clipping
balances advantage contributions and mitigates entropy loss
\citep{xi2026bapo}, while divergence-based controls constrain update size
\citep{qi2026dppo,yao2026drpo}. These mechanisms regulate the update within a
response. \method{} instead decides whether to retain the complete response,
separating response selection from the surrogate applied to retained tokens.

\relatedheading{Sequence-level correction and masking.}
Response-level selection must summarize mismatch across an entire sampled
path. Length-normalized geometric aggregation removes the direct length
dependence of multiplying token ratios \citep{liu2025rlcollapse}; it also
underlies sequence-level clipping in Group Sequence Policy Optimization
(GSPO) \citep{zheng2025gspo}.
DeepSeek-V3.2 explicitly uses off-policy sequence masking based on the mean
signed log-ratio, equivalently the logarithm of a length-normalized
geometric-mean ratio, to stabilize GRPO training \citep{liu2025deepseek}.
Its rule applies a one-sided threshold to negative-advantage responses.
Our GeoMean baseline uses
the same signed-log aggregation with reciprocal-band acceptance
(Eq.~\ref{eq:vanilla_mask}); we examine advantage conditioning separately
in Section~\ref{sec:negative_advantage_ablation}.
TRM takes a
worst-position view, rejecting a response when its maximum token-level
divergence violates a long-horizon trust region \citep{li2026trust}.
\method{} targets a different limitation: geometric aggregation averages
signed log-ratios, so opposing shifts can cancel. Averaging their absolute
values measures bidirectional drift rather than net displacement. The
resulting mask changes which responses contribute to the update while
retaining the token-level policy surrogate.

\section{Method}
\label{sec:method}

\subsection{Problem Setup}
\label{sec:problem_setup}

Let $x$ be a prompt and let
$y_i=(y_{i,1},\ldots,y_{i,T_i})$ be the $i$-th sampled response.
We use \emph{response} and \emph{sequence} interchangeably for these generated tokens.
We consider nonempty responses ($T_i\ge1$) with finite token
log-probabilities, so $0<r_{i,t}<\infty$.
The token-level ratio of the current policy to the rollout policy is
\begin{equation}
r_{i,t}(\theta)
=
\frac{\pi_\theta(y_{i,t}\mid x,y_{i,<t})}
{\pi_{\mathrm{rollout}}(y_{i,t}\mid x,y_{i,<t})}.
\label{eq:ratio}
\end{equation}
For a sequence-level advantage $\hat A_i$, the conventional asymmetric
PPO/GRPO token surrogate is
\begin{equation}
\ell_{i,t}(\theta)
=
\min\Big(
r_{i,t}(\theta)\hat A_i,
\operatorname{clip}\big(
r_{i,t}(\theta),
1-\epsilon_{\mathrm{low}},
1+\epsilon_{\mathrm{high}}
\big)\hat A_i
\Big).
\label{eq:token_surrogate}
\end{equation}
For a binary response mask $M_i\in\{0,1\}$, averaging the retained token
losses gives
\begin{equation}
\mathcal L_{\mathrm{masked}}(\theta)
=
-
\frac{
\sum_i M_i\sum_{t=1}^{T_i}\ell_{i,t}(\theta)
}{
\sum_i M_iT_i
}.
\label{eq:masked_objective}
\end{equation}
The policy-loss contribution is defined as zero if no response is retained.
\method{} specifies the detached mask independently of loss normalization;
Appendix~\ref{app:implementation} states the convention used in the experiments.

\subsection{GeoMean Masking and Ratio Cancellation}
\label{sec:vanilla}

Let $G(y_i)=\exp(\frac{1}{T_i}\sum_t\log r_{i,t})$ denote the geometric-mean
token ratio. We define its log-space deviation and reciprocal-symmetric score as
\begin{equation}
d_{\mathrm{GeoMean}}(y_i)
=
\left|
\frac{1}{T_i}
\sum_{t=1}^{T_i}\log r_{i,t}
\right|,
\qquad
s_{\mathrm{GeoMean}}(y_i)=\exp(d_{\mathrm{GeoMean}}(y_i)).
\label{eq:vanilla_score}
\end{equation}
The response is retained when
\begin{equation}
M_i^{\mathrm{GeoMean}}
=
\ind\!\left[
s_{\mathrm{GeoMean}}(y_i)\le\tau_{\mathrm{GeoMean}}
\right].
\label{eq:vanilla_mask}
\end{equation}
For $\tau_{\mathrm{GeoMean}}\ge1$, this requires the geometric-mean ratio to lie in the
reciprocal interval
$[1/\tau_{\mathrm{GeoMean}},\tau_{\mathrm{GeoMean}}]$.
Equivalently, $d_{\mathrm{GeoMean}}=|\log G|$ and
$s_{\mathrm{GeoMean}}=\max(G,G^{-1})$.
We use GeoMean as a descriptive shorthand for this two-sided rule.
It shares signed-log aggregation with DeepSeek-V3.2,
whose mask instead uses a one-sided threshold on negative-advantage responses
\citep{liu2025deepseek}.

\begin{proposition}[Cancellation under signed aggregation]
\label{prop:cancellation}
For any $B>1$ and any $1<c<B$, there exists an even-length response for
which every token ratio lies outside $[1/c,c]$, while
$s_{\mathrm{GeoMean}}(y)=1$.
\end{proposition}

\begin{proof}
Let the ratios alternate between $B$ and $1/B$.
Every token is outside $[1/c,c]$, but the log-ratios alternate between
$\log B$ and $-\log B$.
Their sum is zero, so $d_{\mathrm{GeoMean}}(y)=0$ and
$s_{\mathrm{GeoMean}}(y)=1$.
\end{proof}

Proposition~\ref{prop:cancellation} does not claim that the exact trajectory
likelihood ratio is incorrect.
Indeed,
$\pi_\theta(y\mid x)/\pi_{\mathrm{rollout}}(y\mid x)=\prod_t r_t$.
It shows that thresholding a signed aggregate does not bound the magnitudes
of the sampled-token log-ratios when used as a response-level filter.
Figure~\ref{fig:null_empirical_case} shows this cancellation in an
actual unfiltered rollout.

\subsection{Cancellation-Aware Response Masking}
\label{sec:carm}

\method{} removes direction before aggregation:
\begin{equation}
d_{\mathrm{CARM}}(y_i)
=
\frac{1}{T_i}
\sum_{t=1}^{T_i}
\left|\log r_{i,t}\right|,
\qquad
s_{\mathrm{CARM}}(y_i)=\exp(d_{\mathrm{CARM}}(y_i)).
\label{eq:carm_score}
\end{equation}
Equivalently, each token contributes the reciprocal-invariant mismatch
$\max(r_{i,t},r_{i,t}^{-1})$ in ratio space, since
$|\log(1/r_{i,t})|=|\log r_{i,t}|$.
The \method{} mask is
\begin{equation}
M_i^{\mathrm{CARM}}
=
\ind\!\left[
s_{\mathrm{CARM}}(y_i)\le\tau_{\mathrm{CARM}}
\right].
\label{eq:carm_mask}
\end{equation}

During optimization, the response mask and token-level surrogate use the
same current-policy forward pass and the stored rollout log-probabilities,
before the optimizer update. The binary gate is treated as a constant during
backpropagation; a rejected response contributes no policy-loss gradient in
that update. Algorithm~\ref{alg:carm} and Appendix~\ref{app:implementation}
give the computation schedule and numerical conventions.

The two scores measure different properties: GeoMean measures the magnitude
of the average signed log-ratio, whereas \method{} measures the average
log-ratio magnitude. Mean-absolute aggregation combines contributions linearly,
preventing cancellation without selecting only the most extreme token or
squaring each deviation. 
Figure~\ref{fig:mechanism} contrasts this rule with GeoMean.

\begin{figure}[t]
    \centering
    \includegraphics[width=0.9\linewidth]{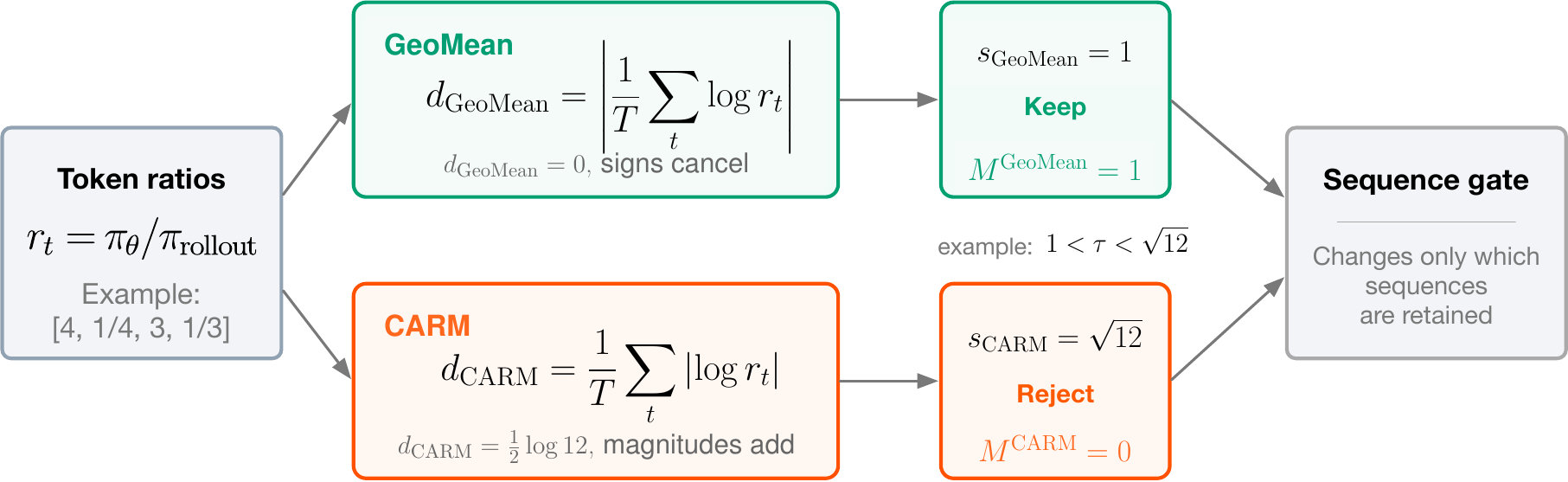}
    \caption{\method{} detects bidirectional token shifts that geometric-mean
    masking can hide. GeoMean averages signed log-ratios before taking an
    absolute value; \method{} averages their magnitudes. The response-level
    gate complements the token-level PPO/GRPO surrogate.}
    \label{fig:mechanism}
\end{figure}

\paragraph{A real disagreement case from an unfiltered rollout.}
We replay both masks on an unfiltered batch of 1,024 responses generated by
Qwen3.5-4B.
With $\tau_{\mathrm{GeoMean}}=1.005$ and $\tau_{\mathrm{CARM}}=1.02$,
160 responses are accepted by both rules, 668 are rejected by both,
195 are accepted only by GeoMean, and one is accepted only by \method{}.
Figure~\ref{fig:null_empirical_case} visualizes a representative GeoMean-only
response whose opposing log-ratio masses cancel;
Appendix~\ref{app:full_disagreement_case} provides all 8,619 tokens.
Its mean absolute log-ratio is $9.79$ times the absolute signed mean, while
the largest single-token contribution accounts for only $0.43\%$ of the
absolute log-ratio mass. The discrepancy is therefore not driven by one
extreme token. These unequal thresholds illustrate different filtering
decisions; the shared-threshold relation is examined separately below.

\begin{proposition}[Acceptance-set inclusion]
\label{prop:refinement}
For every response $y$, $s_{\mathrm{CARM}}(y)\ge s_{\mathrm{GeoMean}}(y)$.
Thus, for any shared threshold $\tau\ge1$,
\begin{equation}
\{y:s_{\mathrm{CARM}}(y)\le\tau\}
\subseteq
\{y:s_{\mathrm{GeoMean}}(y)\le\tau\}.
\label{eq:accepted_set}
\end{equation}
\end{proposition}

\begin{proof}
The triangle inequality gives
\[
\frac1T\sum_t|\log r_t|
\ge
\left|\frac1T\sum_t\log r_t\right|.
\]
Exponentiation preserves the inequality, and the accepted-set inclusion
follows immediately.
\end{proof}

Figure~\ref{fig:null_score_refinement} replays both statistics on the same
1,024 responses from an unfiltered rollout batch, before any
sequence-level gate can alter the training trajectory. Every response obeys
$d_{\mathrm{CARM}}\ge d_{\mathrm{GeoMean}}$, with no numerical violations;
the mean difference is $0.0319$. At every shared threshold, \method{} rejects
at least as many responses as GeoMean, in agreement with
Proposition~\ref{prop:refinement}.

\begin{figure}[t]
    \centering
    \includegraphics[width=\linewidth]{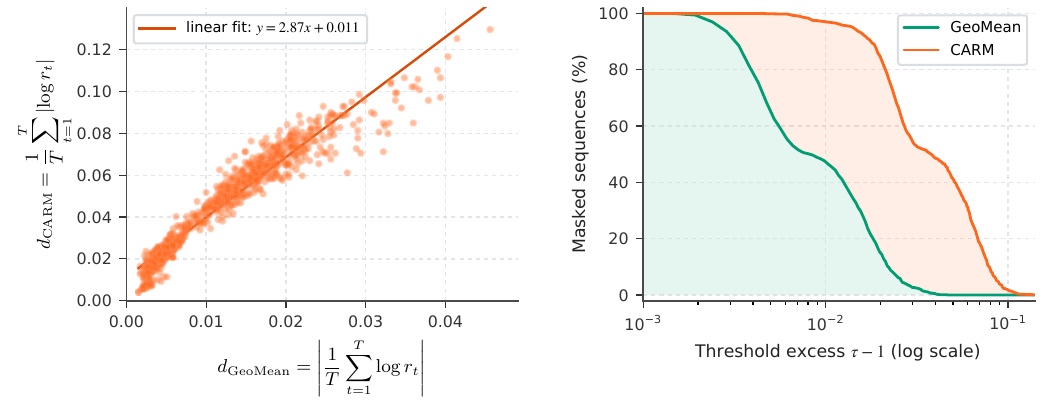}
    \caption{
    Paired deviations (left) and rejection rates at shared thresholds (right)
    on 1,024 unfiltered responses. Every response satisfies
    $d_{\mathrm{CARM}}\ge d_{\mathrm{GeoMean}}$.
    Green shading denotes rejection by both rules; orange shading denotes
    rejection only by \method{}.
    }
    \label{fig:null_score_refinement}
\end{figure}

\subsection{Threshold Interpretation}
\label{sec:threshold_theory}

Consider a token-ratio band $[\ell,u]$ with $0<\ell<1<u$.
Define the lower- and upper-side violation sets and their fractions as
\begin{align}
\mathcal V_-&=\{t:r_t<\ell\},
&p_-&=\frac{|\mathcal V_-|}{T},\\
\mathcal V_+&=\{t:r_t>u\},
&p_+&=\frac{|\mathcal V_+|}{T}.
\end{align}
Let $p_{\mathrm{out}}=p_-+p_+$ and define the log-costs of reaching the
boundaries by $b_-=\log(1/\ell)$ and $b_+=\log u$.
For $p_{\mathrm{out}}>0$, define
\begin{equation}
\bar\delta_{\mathrm{out}}(y)
=
\frac{
\displaystyle
\sum_{t\in\mathcal V_-}\log\frac{\ell}{r_t}
+
\sum_{t\in\mathcal V_+}\log\frac{r_t}{u}
}{|\mathcal V_-|+|\mathcal V_+|},
\label{eq:mean_outside_distance}
\end{equation}
and set $\bar\delta_{\mathrm{out}}(y)=0$ when $p_{\mathrm{out}}=0$.
We call $\bar\delta_{\mathrm{out}}$ the \emph{mean violation severity}: it is
the average log-distance by which an out-of-band token passes its nearest
boundary.
\begin{proposition}[Bounds on ratio-band violations]
\label{prop:asymmetric_budget}
If $s_{\mathrm{CARM}}(y)\le\tau$ with $\tau\ge1$, then
\begin{equation}
p_-b_-+p_+b_+ + p_{\mathrm{out}}\bar\delta_{\mathrm{out}}(y)
\le \log\tau.
\label{eq:frequency_distance_budget}
\end{equation}
Consequently, with $b_{\min}=\min(b_-,b_+)$,
\begin{equation}
p_{\mathrm{out}}
\le
\min\!\left\{1,\frac{\log\tau}{b_{\min}+\bar\delta_{\mathrm{out}}(y)}\right\}
\le\min\!\left\{1,\frac{\log\tau}{b_{\min}}\right\}.
\label{eq:asymmetric_total_frequency}
\end{equation}
\end{proposition}

Thus each response accepted by \method{} has a shared budget for the fraction
of violating token ratios and their mean severity. At a fixed threshold, more severe
violations imply a smaller admissible fraction; the bound does not require a
reciprocal-symmetric band.
This is a deterministic bound on the sampled-token ratios, not on the full
conditional-policy KL divergence.

\begin{proof}
The acceptance condition is
$\frac{1}{T}\sum_t|\log r_t|\le\log\tau$.
For a lower violation,
$|\log r_t|=b_-+\log(\ell/r_t)$; for an upper violation,
$|\log r_t|=b_++\log(r_t/u)$.
Sum these identities over violating tokens, drop the non-negative in-band
contribution, and divide by $T$ to obtain
Equation~\ref{eq:frequency_distance_budget}.
Replacing both boundary costs by $b_{\min}$ gives
$p_{\mathrm{out}}(b_{\min}+\bar\delta_{\mathrm{out}})\le\log\tau$.
Division by the positive denominator and $p_{\mathrm{out}}\le1$ give the
fraction bound. Appendix~\ref{app:detailed_proofs} provides the full derivation.
\end{proof}

\begin{corollary}[Reciprocal-symmetric band]
\label{cor:symmetric}
For $[1/c,c]$ with $c>1$, define
$\mathcal V_c=\{t:r_t<1/c\ \text{or}\ r_t>c\}$,
$p_c=|\mathcal V_c|/T$, and let $\bar\delta_c$ be the mean of
$|\log r_t|-\log c$ over $t\in\mathcal V_c$ (zero when $p_c=0$).
If $s_{\mathrm{CARM}}(y)\le\tau$, then
\begin{equation}
p_c\big(\log c+\bar\delta_c\big)\le\log\tau.
\label{eq:symmetric_budget}
\end{equation}
In particular, $p_c\le\log\tau/\log c$.
\end{corollary}

\begin{proof}
Set $\ell=1/c$ and $u=c$ in
Proposition~\ref{prop:asymmetric_budget}; then $b_-=b_+=\log c$ and
$\bar\delta_{\mathrm{out}}=\bar\delta_c$.
\end{proof}

The experimental PPO/GRPO clip parameters
$\epsilon_{\mathrm{low}}=0.2$ and
$\epsilon_{\mathrm{high}}=0.28$ correspond to the numerical band
$[0.8,1.28]$.
For this band, the lower and upper boundary costs are
$-\log0.8\approx0.2231$ and $\log1.28\approx0.2469$.
Consequently, Equation~\ref{eq:asymmetric_total_frequency} gives
$p_{\mathrm{out}}\le8.88\%$ at $\tau=1.02$ and
$p_{\mathrm{out}}\le17.58\%$ at $\tau=1.04$ for any accepted sequence.
The PPO clipping interval serves only as a reference band for interpreting
violations; \method{} does not impose it as a hard per-token constraint.

\section{Experiments}
\label{sec:experiments}

\subsection{Experimental Setup}
\label{sec:experimental_setup}

Experiments use \texttt{verl} \citep{sheng2025hybridflow} and GRPO. We train
Qwen3.5-4B and Qwen3.5-9B \citep{qwen3.5} on DAPO Math
\citep{yu2025dapo} and evaluate AIME 2024/2025/2026
\citep{maa2026aime} and BeyondAIME \citep{seed2025thinking} with 16 responses
per problem. For code generation, Qwen3.5-9B is trained on the TACO training
split \citep{li2023taco} with executable rewards. Evaluation uses the held-out
TACO test split, LiveCodeBench-v6
\citep{jain2024livecodebench}, and the EvalPlus suites HumanEval+ and MBPP+
\citep{liu2023evalplus}. Metrics are obtained using three random seeds
for both tasks. Math-Verify checks mathematical answers
\citep{kydlicek2025mathverify}; SandboxFusion executes and judges generated
code \citep{fullstackbench2025}. Baselines are Null, IcePop, TRM, and GeoMean.
Null retains the same GRPO procedure, including token-level clipping and
advantage normalization, but applies no additional off-policy response mask.
Appendix~\ref{app:experiment_details} and
Table~\ref{tab:full_configuration} give the datasets, prompts, and configuration.

\subsection{Mathematical Reasoning}
\label{sec:math_results}

Table~\ref{tab:main_results} compares mathematical reasoning across AIME
2024/2025/2026 and BeyondAIME. \method{} achieves the highest four-dataset
average for both Qwen3.5-4B and Qwen3.5-9B. Its three-year AIME average also
exceeds those of Null, IcePop, TRM, and GeoMean: relative to GeoMean, it rises
from $69.17\%$ to $70.42\%$ on Qwen3.5-4B and from $74.03\%$ to $77.50\%$
on Qwen3.5-9B. The latter improvement holds separately on all three AIME
years. For Qwen3.5-4B, \method{} improves AIME 2024, ties on AIME 2025, and
is slightly lower on AIME 2026, while improving the aggregate.

\begin{table}[H]
\centering
\normalsize
\setlength{\tabcolsep}{3pt}
\caption{Mathematical-reasoning mean@16 across methods (\%, higher is better).
\emph{Average} is the mean of AIME 2024, AIME 2025, AIME 2026, and
BeyondAIME. Bold and underline denote the best and second-best values within
each model block.}
\label{tab:main_results}
\begin{tabular}{@{}llccccc@{}}
\toprule
Model & Method & AIME 2024 & AIME 2025 & AIME 2026 & BeyondAIME & Average \\
\midrule
\multirow{5}{*}{Qwen3.5-4B}
& Null & 49.38 & 48.33 & 54.79 & 22.25 & 43.69 \\
& IcePop & 47.08 & 40.83 & 40.42 & 19.31 & 36.91 \\
& TRM & 66.67 & 55.21 & 65.00 & 32.31 & 54.80 \\
& GeoMean & \secondbest{72.50} & \bestresult{64.38} & \bestresult{70.62} & \secondbest{38.88} & \secondbest{61.59} \\
& \methodours & \bestresult{77.08} & \bestresult{64.38} & \secondbest{69.79} & \bestresult{42.56} & \bestresult{63.45} \\
\midrule
\multirow{5}{*}{Qwen3.5-9B}
& Null & 76.67 & 66.88 & 67.71 & 45.25 & 64.13 \\
& IcePop & 70.42 & 59.58 & 66.67 & 34.94 & 57.90 \\
& TRM & 78.33 & \secondbest{67.08} & \secondbest{76.46} & 45.81 & 66.92 \\
& GeoMean & \secondbest{80.42} & 66.67 & 75.00 & \secondbest{47.06} & \secondbest{67.29} \\
& \methodours & \bestresult{82.08} & \bestresult{71.04} & \bestresult{79.38} & \bestresult{49.19} & \bestresult{70.42} \\
\bottomrule
\end{tabular}
\end{table}

On BeyondAIME, \method{} raises mean@16 from $38.88\%$ to $42.56\%$ on
Qwen3.5-4B and from $47.06\%$ to $49.19\%$ on Qwen3.5-9B relative to GeoMean.
The gains therefore extend beyond the annual AIME benchmarks. Across all
four datasets, the Average increases from $61.59\%$ to $63.45\%$ on
Qwen3.5-4B and from $67.29\%$ to $70.42\%$ on Qwen3.5-9B. For the larger
model, the aggregate lead accompanies an improvement on every benchmark;
for the smaller model, it reflects an overall gain with a mixed year-by-year
AIME pattern.

Relative to TRM, \method{} improves the three-year AIME average by
$8.13$ points on Qwen3.5-4B and $3.54$ points on Qwen3.5-9B despite TRM's
higher rejection rates (Section~\ref{sec:sequence_utilization}).
Relative to IcePop, \method{} improves the four-dataset Average by
$26.54$ points on Qwen3.5-4B and $12.52$ points on Qwen3.5-9B.

\subsection{Code Generation}
\label{sec:code_results}

Table~\ref{tab:code_results} reports pass@1 for Qwen3.5-9B trained on the
TACO training split. Testing on TACO measures performance on held-out
problems from the same dataset, while LiveCodeBench-v6, HumanEval+, and MBPP+
evaluate the resulting policy on three additional benchmarks.

\begin{table}[H]
\centering
\normalsize
\caption{Evaluation pass@1 across code benchmarks (\%, higher is better).
TACO denotes the held-out test split, separate from the training split.
\emph{Average} is the mean of the four
datasets. Bold and underline denote the best and second-best values.}
\label{tab:code_results}
\setlength{\tabcolsep}{3.2pt}
\begin{tabular}{@{}lccccc@{}}
\toprule
Method & TACO & LiveCodeBench-v6 & HumanEval+ & MBPP+ & Average \\
\midrule
Null & 34.50 & 38.46 & 85.57 & 75.22 & 58.44 \\
IcePop & 45.87 & 43.36 & \secondbest{92.68} & \secondbest{79.81} & \secondbest{65.43} \\
TRM & 36.60 & \secondbest{44.97} & 87.80 & 76.72 & 61.52 \\
GeoMean & \secondbest{46.03} & 44.50 & 92.28 & 78.31 & 65.28 \\
\methodours & \bestresult{46.70} & \bestresult{51.20} & \bestresult{94.11} & \bestresult{81.22} & \bestresult{68.31} \\
\bottomrule
\end{tabular}
\end{table}

\method{} improves average pass@1 by $3.03$ points over GeoMean and $2.88$
points over IcePop, the strongest baseline on average. The gains are not
uniform: relative to GeoMean, the increase is modest on TACO, from
$46.03\%$ to $46.70\%$, and largest on LiveCodeBench-v6, from $44.50\%$ to
$51.20\%$. On HumanEval+ and MBPP+, \method{} reaches $94.11\%$ and
$81.22\%$, also exceeding IcePop, the strongest baseline on these suites.
Thus, the improvement is not confined to held-out TACO problems.

\section{Analysis}
\label{sec:analysis}
\raggedbottom

We examine whether the gains are explained by filtering more responses,
whether advantage sign should override the gate, and whether improvements
persist across different sampling criteria.

\subsection{Update Redundancy and Effective Filtering}
\label{sec:sequence_utilization}

\begin{wrapfigure}{r}{0.48\linewidth}
    \centering
    \includegraphics[width=\linewidth]{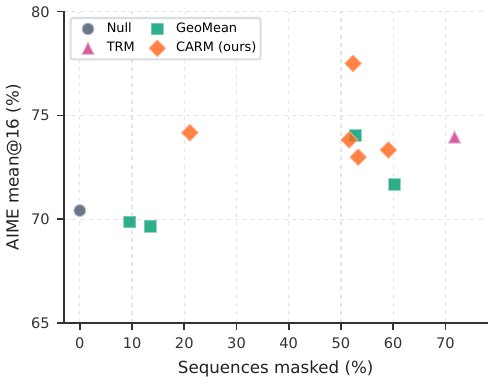}
    \caption{Sequence filtering versus AIME accuracy.
    Each point is a masking configuration. AIME mean@16 averages the
    2024/2025/2026 benchmarks.
    }
    \label{fig:accuracy_vs_mask_fraction}
\end{wrapfigure}

We use \emph{update redundancy} to describe improved validation accuracy
despite omitting response-level updates, without implying that rejected
responses lack useful information.

GeoMean and \method{} mask an average of $55.65\%$ and $54.52\%$ of
sequences during Qwen3.5-4B training, yet improve three-year AIME mean@16 over
Null by $18.33$ and $19.58$ points. For Qwen3.5-9B, the corresponding
fractions are $52.75\%$ and $52.32\%$, with gains of $3.61$ and $7.08$ points.
Filtering persists throughout training (Appendix Figure~\ref{fig:masked_fraction}).
Filtering already-generated responses complements selective rollout use
and data-efficient reasoning
\citep{yu2025dapo,sun2025prepo,lu2026rolloutbandits,wang2025oneshotrlvr,zheng2025greso}.

\label{sec:threshold_sensitivity}

Figure~\ref{fig:accuracy_vs_mask_fraction} compares AIME accuracy with mean
sequence-masked fraction.
On Qwen3.5-9B, \method{} outperforms GeoMean by $3.47$
points while their masked fractions differ by only $0.43$ points. Accuracy
is not monotone in the masked fraction, and TRM rejects more responses
without outperforming \method{}.

Paired-score differences (Figure~\ref{fig:null_score_refinement}) show that
rejection rates alone do not characterize retained trajectories.
Tightening \method{}'s frequency--severity budget
(Proposition~\ref{prop:asymmetric_budget}) limits sampled-token mismatch but
removes update contributions, consistent with the observed non-monotonicity.

\Needspace{7\baselineskip}
\subsection{Advantage-Conditional Masking Ablation}
\label{sec:negative_advantage_ablation}

DeepSeek-V3.2 masks only negative-advantage responses \citep{liu2025deepseek}.
We test this conditioning on Qwen3.5-9B with \method{}'s score and threshold
fixed within each task. Standard masking reaches $75.63\%$ AIME mean@16
(2024/2025/2026) versus $60.28\%$ for negative-only masking; LiveCodeBench-v6
scores are $50.58\%$ versus $42.22\%$ (Figure~\ref{fig:negative_only_ablation}).

Even a correctly rewarded response can be stale if the current policy assigns
substantially different probabilities to its tokens. PPO clips positive-advantage
terms above the upper ratio bound, but other token updates can remain active
\citep{schulman2017ppo}; it does not enforce a response-level trust region
\citep{li2026trust}. These results suggest that response-level mismatch checks
can remain useful even for positive-advantage responses.

\Needspace{8\baselineskip}
\subsection{Sampling Consistency}
\label{sec:sampling_consistency}

We assess sampling consistency using bootstrapped best-of-$k$ and worst-of-$k$:
the former requires at least one correct response among $k$, while the latter
requires all $k$ to be correct. Worst-of-$k$ measures sampling repeatability,
not distribution-shift or adversarial robustness. Appendix~\ref{app:sampling_scaling}
defines the estimators and reports results for $k\in\{1,2,4,8,16\}$.

For Qwen3.5-9B, Figure~\ref{fig:sampling_scaling_9b} averages results over
AIME 2024/2025/2026. \method{} improves mean@16, best@16, and worst@16 over
GeoMean by $3.47$, $2.74$, and $2.78$ points, respectively; the corresponding
gains over TRM are $3.54$, $0.30$, and $4.61$ points. \method{} also achieves
the highest best-of-$k$ and worst-of-$k$ values throughout the displayed range.

\begin{figure}[H]
    \centering
    \begin{minipage}[t]{0.36\linewidth}
    \vspace{0pt}
    \centering
    \includegraphics[width=\linewidth]{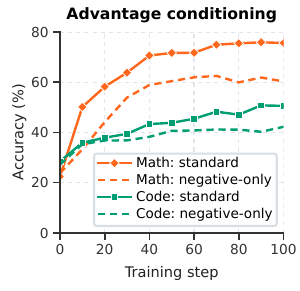}
    \caption{Standard (solid) versus negative-only (dashed) \method{} masking.}
    \label{fig:negative_only_ablation}
    \end{minipage}\hfill
    \begin{minipage}[t]{0.62\linewidth}
    \vspace{0pt}
    \centering
    \includegraphics[width=\linewidth]{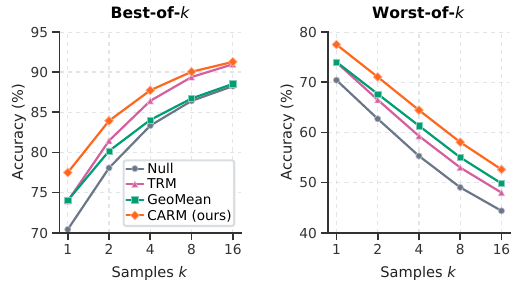}
    \caption{Bootstrapped best-of-$k$ (left) and worst-of-$k$
    (right) accuracy on AIME.}
    \label{fig:sampling_scaling_9b}
    \end{minipage}
\end{figure}

\section{Conclusion}
\label{sec:conclusion}

We proposed \method, a response-level off-policy mask that 
averages absolute token log-ratios to
prevent opposing probability changes from canceling. Our analysis shows that signed aggregation
can accept responses even when every sampled-token ratio lies outside a reciprocal-symmetric band.
In contrast, \method bounds the frequency and severity of ratio-band violations through a joint budget
for each accepted response, giving the threshold a token-level interpretation. Experiments show
gains in mathematical reasoning and code generation over the evaluated baselines, together with improved sampling consistency. 
These findings support \method{} as a principled response-selection
criterion that complements reward-directed PPO/GRPO optimization.

\clearpage
\bibliography{references}
\bibliographystyle{iclr2027_conference}

\clearpage
\appendix
\raggedbottom
\section{Detailed Proofs}
\label{app:detailed_proofs}

\subsection{Proof of Proposition~\ref{prop:cancellation}}

\paragraph{Step 1: construct an even-length response.}
Fix arbitrary constants $B>1$ and $1<c<B$.
Choose an arbitrary positive integer $m$ and let the response length be
$T=2m$. For every pair $j\in\{1,\ldots,m\}$, define
\begin{equation}
r_{2j-1}=B,
\qquad
r_{2j}=\frac1B.
\label{eq:appendix_cancellation_construction}
\end{equation}
Thus the response contains exactly $m$ ratios equal to $B$ and exactly $m$
ratios equal to $1/B$.
These ratios can arise from valid conditional policies: at each odd
position, assign the sampled token probabilities $1/(1+B)$ under the rollout
policy and $B/(1+B)$ under the current policy, and swap them at each even
position. Assign the remaining probability to a second token. Both policies
are normalized, all probabilities are positive, and their sampled-token
ratios are exactly those in Equation~\ref{eq:appendix_cancellation_construction}.

\paragraph{Step 2: verify that every ratio violates the band.}
For each odd position, $r_{2j-1}=B>c$, so it lies strictly above the upper
endpoint of $[1/c,c]$. For each even position, $1<c<B$ and all quantities are
positive. Taking reciprocals reverses the inequalities and gives
\begin{equation}
1>\frac1c>\frac1B.
\end{equation}
Hence $r_{2j}=1/B<1/c$, so it lies strictly below the lower endpoint.
Therefore every one of the $T$ token ratios is outside $[1/c,c]$.

\paragraph{Step 3: evaluate the GeoMean statistic.}
The logarithm of a reciprocal satisfies
$\log(1/B)=-\log B$. Grouping the signed log-ratios into the $m$ adjacent
pairs in Equation~\ref{eq:appendix_cancellation_construction} yields
\begin{align}
\sum_{t=1}^{T}\log r_t
&=\sum_{j=1}^{m}\left(\log B+\log(1/B)\right)\\
&=\sum_{j=1}^{m}\left(\log B-\log B\right)\\
&=\sum_{j=1}^{m}0\\
&=0.
\end{align}
Substitution into Equation~\ref{eq:vanilla_score} gives
\begin{align}
d_{\mathrm{GeoMean}}(y)
&=\left|\frac1T\sum_{t=1}^{T}\log r_t\right|\\
&=\left|\frac{0}{T}\right|\\
&=0,
\end{align}
and therefore
\begin{equation}
s_{\mathrm{GeoMean}}(y)=\exp(d_{\mathrm{GeoMean}}(y))=\exp(0)=1.
\end{equation}
This proves the proposition: every token can violate the prescribed reciprocal
band even though GeoMean reports its minimum possible score.

\paragraph{Step 4: evaluate the CARM statistic for comparison.}
Because $B>1$, $\log B>0$, and hence
$|\log B|=|\log(1/B)|=\log B$. Consequently,
\begin{align}
d_{\mathrm{CARM}}(y)
&=\frac1T\sum_{t=1}^{T}|\log r_t|\\
&=\frac1{2m}\sum_{j=1}^{m}
\left(|\log B|+|\log(1/B)|\right)\\
&=\frac1{2m}\sum_{j=1}^{m}2\log B\\
&=\log B,
\end{align}
so $s_{\mathrm{CARM}}(y)=\exp(\log B)=B$.
Thus any GeoMean threshold at least one accepts this construction, whereas
\method{} rejects it whenever $\tau_{\mathrm{CARM}}<B$.

\subsection{Proof of Proposition~\ref{prop:refinement}}

\paragraph{Step 1: move to log-ratio coordinates.}
For a fixed response of positive length $T$, set
$z_t=\log r_t\in\mathbb R$. For $T=1$, both scores are identical.
For $T\ge2$, the triangle inequality applied repeatedly to the finite sum gives
\begin{align}
\left|\sum_{t=1}^{T}z_t\right|
&=|z_1+(z_2+\cdots+z_T)|\\
&\le |z_1|+|z_2+\cdots+z_T|\\
&\le\cdots\le\sum_{t=1}^{T}|z_t|.
\end{align}

\paragraph{Step 2: normalize by response length.}
Since $T>0$, division by $T$ preserves the direction of the inequality:
\begin{equation}
\left|\frac1T\sum_t z_t\right|
\le \frac1T\sum_t|z_t|.
\end{equation}
Using $z_t=\log r_t$ and Equations~\ref{eq:vanilla_score} and
\ref{eq:carm_score}, the left-hand side is exactly
$d_{\mathrm{GeoMean}}(y)$ and the right-hand side is exactly
$d_{\mathrm{CARM}}(y)$. Therefore
\begin{equation}
d_{\mathrm{GeoMean}}(y)\le d_{\mathrm{CARM}}(y).
\end{equation}

\paragraph{Step 3: return to ratio space.}
The exponential function is strictly increasing on $\mathbb R$. Applying it to
both sides gives
\begin{align}
\exp(d_{\mathrm{GeoMean}}(y))
&\le \exp(d_{\mathrm{CARM}}(y)),\\
s_{\mathrm{GeoMean}}(y)
&\le s_{\mathrm{CARM}}(y).
\end{align}
This proves the pointwise score inequality.

\paragraph{Step 4: prove accepted-set inclusion.}
Take an arbitrary response
$y\in\{y:s_{\mathrm{CARM}}(y)\le\tau\}$. By membership in this set,
$s_{\mathrm{CARM}}(y)\le\tau$. Combining this fact with the pointwise result
from Step 3 gives
\begin{equation}
s_{\mathrm{GeoMean}}(y)
\le s_{\mathrm{CARM}}(y)
\le\tau.
\end{equation}
Thus $y\in\{y:s_{\mathrm{GeoMean}}(y)\le\tau\}$. Since the argument holds for
every element of the CARM accepted set, Equation~\ref{eq:accepted_set}
follows.

\paragraph{Equality condition.}
For real scalars, equality in the triangle inequality holds exactly when all
nonzero $z_t$ have the same sign. Accordingly,
$s_{\mathrm{GeoMean}}=s_{\mathrm{CARM}}$ when all nonzero token log-ratios are
nonnegative or all are nonpositive. If at least one $z_t>0$ and at least one
$z_{t'}<0$, the inequality is strict.

\subsection{Proof of Proposition~\ref{prop:asymmetric_budget}}

\paragraph{Step 1: convert acceptance into an additive budget.}
By definition,
\begin{equation}
s_{\mathrm{CARM}}(y)
=\exp\left(\frac1T\sum_{t=1}^{T}|\log r_t|\right).
\end{equation}
The acceptance assumption $s_{\mathrm{CARM}}(y)\le\tau$ and monotonicity of
$\log(\cdot)$ imply
\begin{align}
\log s_{\mathrm{CARM}}(y)
&\le\log\tau,\\
\frac1T\sum_{t=1}^{T}|\log r_t|
&\le\log\tau.
\end{align}
Multiplying by the positive length $T$ gives the total absolute-log budget
\begin{equation}
\sum_{t=1}^{T}|\log r_t|\le T\log\tau.
\label{eq:appendix_total_budget}
\end{equation}

\paragraph{Step 2: partition positions and define the boundary costs.}
Let $\mathcal V_0=\{1,\ldots,T\}\setminus(\mathcal V_-\cup\mathcal V_+)$.
Because $\ell<1<u$, the sets $\mathcal V_-$, $\mathcal V_0$, and
$\mathcal V_+$ are disjoint and cover all token positions. Define
\begin{equation}
b_-=\log\frac1\ell=-\log\ell,
\qquad
b_+=\log u.
\end{equation}
The assumptions $0<\ell<1<u$ imply $b_->0$ and $b_+>0$.

\paragraph{Step 3: decompose a lower-side violation exactly.}
Take $t\in\mathcal V_-$. Then $0<r_t<\ell<1$, so
$|\log r_t|=-\log r_t$. Inserting and subtracting $\log\ell$ gives
\begin{align}
|\log r_t|
&=-\log r_t\\
&=-\log\ell+\log\ell-\log r_t\\
&=b_-+\log\frac{\ell}{r_t}.
\label{eq:appendix_lower_exact}
\end{align}
Since $\ell/r_t>1$, the final term is positive. It approaches zero at the
lower boundary and grows with the token's log-distance below that boundary.

\paragraph{Step 4: decompose an upper-side violation exactly.}
Take $t\in\mathcal V_+$. Then $r_t>u>1$, so
$|\log r_t|=\log r_t$. Inserting and subtracting $\log u$ gives
\begin{align}
|\log r_t|
&=\log r_t\\
&=\log u+\log r_t-\log u\\
&=b_++\log\frac{r_t}{u}.
\label{eq:appendix_upper_exact}
\end{align}
Because $r_t/u>1$, this final term is likewise a positive log-distance
past the upper boundary.

\paragraph{Step 5: sum the exact decompositions.}
Partition the total absolute-log drift and substitute
Equations~\ref{eq:appendix_lower_exact} and
\ref{eq:appendix_upper_exact} term by term:
\begin{align}
\sum_{t=1}^{T}|\log r_t|
=&\ |\mathcal V_-|b_-+|\mathcal V_+|b_+\\
&+\sum_{t\in\mathcal V_-}\log\frac{\ell}{r_t}
+\sum_{t\in\mathcal V_+}\log\frac{r_t}{u}\\
&+\sum_{t\in\mathcal V_0}|\log r_t|.
\label{eq:appendix_partition_exact}
\end{align}

\paragraph{Step 6: identify the conditional mean log-distance.}
If $p_{\mathrm{out}}>0$, Equation~\ref{eq:mean_outside_distance} implies
\begin{align}
&\sum_{t\in\mathcal V_-}\log\frac{\ell}{r_t}
+\sum_{t\in\mathcal V_+}\log\frac{r_t}{u}\\
&\qquad=
\big(|\mathcal V_-|+|\mathcal V_+|\big)
\bar\delta_{\mathrm{out}}(y).
\end{align}
Divide by $T$ and use
$(|\mathcal V_-|+|\mathcal V_+|)/T=p_{\mathrm{out}}$:
\begin{equation}
\frac1T\left(
\sum_{t\in\mathcal V_-}\log\frac{\ell}{r_t}
+\sum_{t\in\mathcal V_+}\log\frac{r_t}{u}
\right)
=p_{\mathrm{out}}\bar\delta_{\mathrm{out}}(y).
\label{eq:appendix_distance_identity}
\end{equation}
If $p_{\mathrm{out}}=0$, both sums are empty; our zero convention makes the
same identity valid.

\paragraph{Step 7: obtain the master budget.}
The in-band sum in Equation~\ref{eq:appendix_partition_exact} is nonnegative.
Drop it, apply Equation~\ref{eq:appendix_total_budget}, divide by $T$, and use
Equation~\ref{eq:appendix_distance_identity}:
\begin{equation}
p_-b_-+p_+b_+
+p_{\mathrm{out}}\bar\delta_{\mathrm{out}}(y)
\le\log\tau.
\end{equation}
This is Equation~\ref{eq:frequency_distance_budget}.

\paragraph{Step 8: bound the total violation fraction.}
Let $b_{\min}=\min(b_-,b_+)>0$. Keeping the severity contribution in the
master budget gives
\begin{align}
p_{\mathrm{out}}(b_{\min}+\bar\delta_{\mathrm{out}})
&=(p_-+p_+)b_{\min}+p_{\mathrm{out}}\bar\delta_{\mathrm{out}}\\
&\le p_-b_-+p_+b_++p_{\mathrm{out}}\bar\delta_{\mathrm{out}}\\
&\le\log\tau.
\end{align}
Because $b_{\min}+\bar\delta_{\mathrm{out}}>0$, division yields
\begin{equation}
p_{\mathrm{out}}\le\frac{\log\tau}{b_{\min}+\bar\delta_{\mathrm{out}}}
\le\frac{\log\tau}{b_{\min}}.
\end{equation}
Combining each inequality with $p_{\mathrm{out}}\le1$ proves
Equation~\ref{eq:asymmetric_total_frequency}, including the case
$p_{\mathrm{out}}=0$ under the stated zero convention.

\paragraph{Step 9: recover the frequency-only budget.}
Every term defining $\bar\delta_{\mathrm{out}}$ is nonnegative. Dropping the
distance term from the master budget therefore gives
$p_-b_-+p_+b_+\le\log\tau$.
Let $b_{\min}=\min(b_-,b_+)>0$. Then
\begin{align}
p_{\mathrm{out}}b_{\min}
&=(p_-+p_+)b_{\min}\\
&\le p_-b_-+p_+b_+\\
&\le\log\tau.
\end{align}
Division by $b_{\min}$ gives
$p_{\mathrm{out}}\le\log\tau/b_{\min}$. Independently,
$p_{\mathrm{out}}\le1$ because it is a fraction. Taking the smaller upper
bound recovers the severity-independent consequence of
Equation~\ref{eq:asymmetric_total_frequency}.

\subsection{Proof of Corollary~\ref{cor:symmetric}}

\paragraph{Step 1: specialize the two endpoints.}
Set $\ell=1/c$ and $u=c$. Because $c>1$, the assumptions
$0<\ell<1<u$ hold, and the two boundary costs are identical:
\begin{align}
b_-&=\log(1/\ell)=\log c,\\
b_+&=\log u=\log c.
\end{align}

\paragraph{Step 2: combine the violation sets.}
The lower and upper violation sets are
$\mathcal V_-=\{t:r_t<1/c\}$ and
$\mathcal V_+=\{t:r_t>c\}$. They are disjoint because $c>1$, and
their union is exactly $\mathcal V_c$. Hence
\begin{equation}
p_-+p_+=\frac{|\mathcal V_c|}{T}=p_c.
\end{equation}

\paragraph{Step 3: identify the reciprocal-band log-distance.}
For $t\in\mathcal V_-$,
\begin{align}
\log\frac{\ell}{r_t}
&=\log\frac{1/c}{r_t}\\
&=|\log r_t|-\log c,
\end{align}
where the second equality uses $r_t<1$. For $t\in\mathcal V_+$,
\begin{align}
\log\frac{r_t}{u}
&=\log r_t-\log c\\
&=|\log r_t|-\log c,
\end{align}
where $r_t>1$. Thus the conditional mean in
Equation~\ref{eq:mean_outside_distance} is precisely $\bar\delta_c$.

\paragraph{Step 4: simplify the master budget.}
Substitute $p_-+p_+=p_c$, $b_-=b_+=\log c$, and
$\bar\delta_{\mathrm{out}}=\bar\delta_c$ into
Equation~\ref{eq:frequency_distance_budget}:
\begin{align}
p_-\log c+p_+\log c+p_c\bar\delta_c
&\le\log\tau,\\
p_c\log c+p_c\bar\delta_c
&\le\log\tau,\\
p_c(\log c+\bar\delta_c)
&\le\log\tau.
\end{align}
This is Equation~\ref{eq:symmetric_budget}. Finally,
$\bar\delta_c\ge0$, so dropping $p_c\bar\delta_c$ gives
$p_c\log c\le\log\tau$. Since $\log c>0$, division by $\log c$ yields
$p_c\le\log\tau/\log c$.

\FloatBarrier
\section{Implementation of CARM}
\label{app:implementation}

Algorithm~\ref{alg:carm} gives the response-level masking rule, treating
the hard threshold decision as constant during backpropagation. The gate
changes which response losses contribute to an update while retaining the
token-level PPO/GRPO surrogate. 

\begin{algorithm}[H]
\caption{\method{} response-level off-policy masking}
\label{alg:carm}
\begin{algorithmic}[1]
\Require Finite current and rollout token log-probabilities
$\log\pi_\theta(y_{i,t}\mid x_i,y_{i,<t})$ and
$\log\pi_{\mathrm{rollout}}(y_{i,t}\mid x_i,y_{i,<t})$; response-token mask
$a_{i,t}\in\{0,1\}$ with $\sum_t a_{i,t}\ge1$ for each response;
threshold $\tau_{\mathrm{CARM}}\ge1$
\Ensure Response-acceptance indicators $M_i^{\mathrm{CARM}}$
\For{each response $y_i$}
    \State $T_i \gets \sum_t a_{i,t}$
    \State $\Delta_{i,t} \gets
    \operatorname{stopgrad}\!\left(\log\pi_\theta(y_{i,t}\mid\cdot)
    -\log\pi_{\mathrm{rollout}}(y_{i,t}\mid\cdot)\right)$
    \State $d_{\mathrm{CARM},i} \gets
    \frac{1}{T_i}\sum_t a_{i,t}|\Delta_{i,t}|$
    \State $s_{\mathrm{CARM},i} \gets \exp(d_{\mathrm{CARM},i})$
    \State $M_i^{\mathrm{CARM}} \gets
    \ind[s_{\mathrm{CARM},i}\le\tau_{\mathrm{CARM}}]$
\EndFor
\State Gate each per-response PPO/GRPO loss by $M_i^{\mathrm{CARM}}$ and
apply the training implementation's loss normalization.
\end{algorithmic}
\end{algorithm}

\paragraph{Mask computation schedule.}
We store the sampled-token log-probabilities produced during rollout.
During training, each optimizer mini-batch is divided into micro-batches;
each micro-batch forward supplies the current log-probabilities used by both
the response mask and the token surrogate. The optimizer updates the policy
after gradient accumulation over the mini-batch. Threshold comparisons
produce a binary gate with no gradient through the decision, as represented
by $\operatorname{stopgrad}$ in Algorithm~\ref{alg:carm}.
Masking is local to loss evaluation rather than permanent removal from the rollout batch. 
Our experiments use one optimization epoch per rollout batch;
with multiple epochs, the gate would be recomputed whenever a response is evaluated under the updated policy.

The PPO/GRPO loss implementation in \texttt{verl} also applies dual clipping
to negative advantages. Writing $\ell_{i,t}$ for
Equation~\ref{eq:token_surrogate}, its surrogate is
\begin{equation}
\ell^{\mathrm{dual}}_{i,t}
=
\begin{cases}
\max\{\ell_{i,t},c_{\mathrm{dual}}\hat A_i\},&\hat A_i<0,\\
\ell_{i,t},&\hat A_i\ge0,
\end{cases}
\qquad c_{\mathrm{dual}}=3.
\end{equation}
The sequence mask multiplies this surrogate without changing its clipping
rule or normalization.

\paragraph{Numerical convention.}
The theoretical scores use exact log-ratios. The implementation uses
\texttt{verl}'s numerical guard
$\widetilde z_t=\operatorname{clip}(\log r_t,-20,20)$ when computing the
mask and token surrogate. If the guard is inactive, the implemented and
theoretical scores coincide. Otherwise, the score inequalities and
frequency--severity budget apply to the effective ratios
$\widetilde r_t=\exp(\widetilde z_t)$.
For a band strictly inside $[e^{-20},e^{20}]$, including $[0.8,1.28]$,
clipping does not change which tokens violate the band. Thus the
frequency-only guarantee also holds for the original ratios, while the
severity term measures the guarded log-distances.

\FloatBarrier
\clearpage
\section{Complete Response for the Disagreement Case}
\label{app:full_disagreement_case}

Figure~\ref{fig:null_empirical_case} in the main paper visualizes the complete signed-drift
trajectory and a tokenizer-aligned local window. For auditability, the five
pages below show all 8,619 response tokens from the same rollout without
elision. Each token is rendered in its tokenizer-aligned box: red denotes
$\log r_t>0$, green denotes $\log r_t<0$, and color intensity encodes
$|\log r_t|$. The first-page header repeats the response-level scores and masking
decisions.
For this response, $d_{\mathrm{GeoMean}}=0.002502$ and
$d_{\mathrm{CARM}}=0.024486$: signed aggregation cancels $89.78\%$ of the
absolute log-ratio mass. The example demonstrates score cancellation and differing mask
decisions.

\begin{center}
\includegraphics[page=1,width=0.96\linewidth]{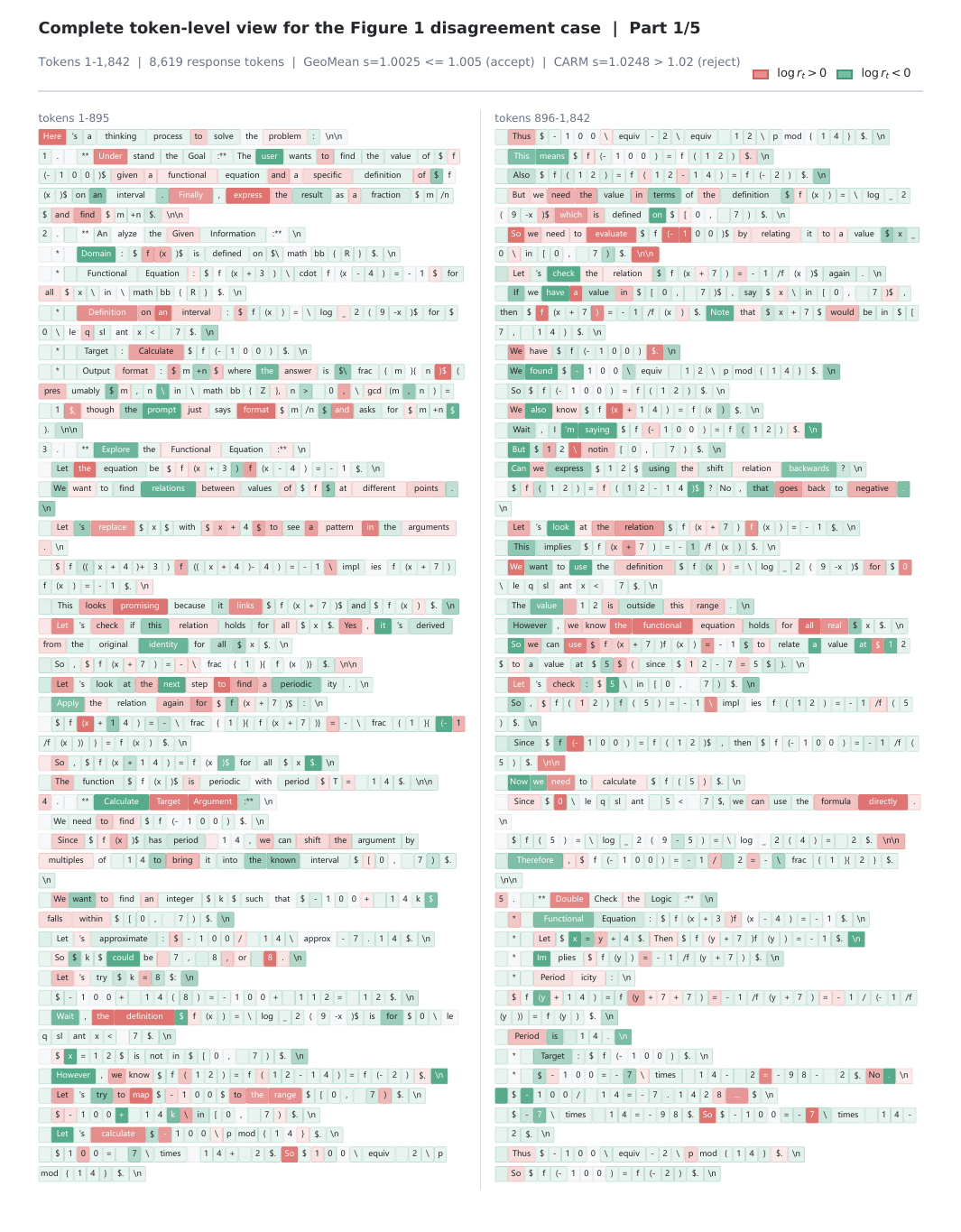}
\end{center}
\clearpage
\begin{center}
\includegraphics[page=2,width=0.96\linewidth]{figures/empirical_cancellation_full_sequence.pdf}
\end{center}
\clearpage
\begin{center}
\includegraphics[page=3,width=0.96\linewidth]{figures/empirical_cancellation_full_sequence.pdf}
\end{center}
\clearpage
\begin{center}
\includegraphics[page=4,width=0.96\linewidth]{figures/empirical_cancellation_full_sequence.pdf}
\end{center}
\clearpage
\begin{center}
\includegraphics[page=5,width=0.96\linewidth]{figures/empirical_cancellation_full_sequence.pdf}
\end{center}

\FloatBarrier
\section{Experimental Settings}
\label{app:experiment_details}

\subsection{Data}

\paragraph{Math training data.}
The mathematical-reasoning experiments use the competition-style problems in
the DAPO Math training corpus \citep{yu2025dapo}.
Math-Verify~\citep{kydlicek2025mathverify}
extracts the boxed final answer and checks its equivalence to the reference.
Null, IcePop, TRM, GeoMean,
and \method{} use the same prompt order, rollout group size, rewards, optimizer
settings, and token-level PPO/GRPO clipping; only the off-policy correction
differs.

\paragraph{Math evaluation data.}
AIME 2024, AIME 2025, and AIME 2026 each contain 30 contest problems and are
reported separately \citep{maa2026aime}. BeyondAIME contributes 100 additional
problems designed to extend evaluation beyond the three annual AIME sets
\citep{seed2025thinking}. We sample 16 responses
per problem and compute \texttt{acc/mean@16} from exact final-answer
correctness. The Average in Table~\ref{tab:main_results} gives AIME 2024,
AIME 2025, AIME 2026, and BeyondAIME equal weight. Evaluation is performed
before training and periodically during optimization.

\paragraph{Code training data.}
The general-code experiments use problems from the TACO training split
\citep{li2023taco}, with natural-language task
descriptions, input/output specifications, and executable tests. During RL,
the reward checks up to the first ten available tests for each training
problem. This fixed cap keeps execution cost and sandbox load comparable
across methods; held-out benchmark tests are not used by the training reward.
All five methods receive the same prompts, decoding budget, executable reward,
and optimizer configuration.

\paragraph{Code evaluation data.}
We evaluate each frozen TACO-trained checkpoint on four complementary suites.
TACO evaluation uses 1,000 problems from its held-out test split and up to the first ten
tests per problem. LiveCodeBench-v6 contains 1,054 temporally collected coding
problems and is judged with its complete released tests. HumanEval+ and MBPP+
contain 164 and 378 problems, respectively, and use the stronger EvalPlus test
suites rather than only the original base tests. MBPP+ extends the MBPP
benchmark \citep{austin2021program}. Every generated program is
rejudged with SandboxFusion~\citep{fullstackbench2025}, with functional
correctness reported as pass@1.

\subsection{Prompt Templates}

Figures~\ref{fig:math_prompt_template} and~\ref{fig:code_prompt_template}
reproduce the task-specific chat messages used throughout training and
evaluation. Dataset problem statements replace the placeholders without
changing the surrounding instructions.

\begin{figure}[H]
\centering
\setlength{\fboxsep}{8pt}
\fbox{%
\begin{minipage}{0.93\linewidth}
\normalsize
\raggedright
\textbf{System:} A conversation between User and Assistant. The user asks a
question, and the Assistant solves it.\par
\textbf{User:} Please solve the problem step by step, and put your final
answer within \texttt{\textbackslash boxed\{\}}.\\
\texttt{<problem>}
\end{minipage}%
}
\caption{Prompt template for mathematical-reasoning experiments. The
placeholder \texttt{<problem>} is replaced by the problem text.}
\label{fig:math_prompt_template}
\end{figure}

\begin{figure}[H]
\centering
\setlength{\fboxsep}{8pt}
\fbox{%
\begin{minipage}{0.93\linewidth}
\normalsize
\raggedright
\textbf{System:} You are an expert competitive programming assistant.\par
Solve the user's programming problem in Python 3. You may reason about the
solution, algorithm, and edge cases. At the end, provide the final answer as
exactly one fenced Python code block.\par
The final code must be complete and executable. It should read from standard
input and write to standard output, unless the problem provides starter code
requiring a specific function or class interface.\par
\textbf{User:} Solve the following programming problem.\par
Problem:\\
\texttt{<problem>}\\
\texttt{<starter code, when provided>}\par
Return your final solution in this format:\\
\texttt{\textasciigrave{}\textasciigrave{}\textasciigrave{}python}\\
\texttt{\# your complete Python solution}\\
\texttt{\textasciigrave{}\textasciigrave{}\textasciigrave{}}
\end{minipage}%
}
\caption{Prompt template for code-generation experiments. The placeholder
\texttt{<problem>} is replaced by the problem text; the starter-code field is
included only when provided by the dataset.}
\label{fig:code_prompt_template}
\end{figure}

\subsection{Training and Evaluation Configuration}

Table~\ref{tab:full_configuration} consolidates the shared settings,
method-specific controls, and threshold grids used to produce the reported
comparisons.

\begin{table}[H]
\centering
\normalsize
\caption{Complete experimental configuration for mathematical reasoning and
code generation.}
\label{tab:full_configuration}
\begingroup
\renewcommand{\arraystretch}{0.96}
\begin{tabular}{@{}>{\raggedright\arraybackslash}p{0.29\linewidth}
                    >{\raggedright\arraybackslash}p{0.66\linewidth}@{}}
\toprule
Field & Value \\
\midrule
Base models & Qwen3.5-4B and Qwen3.5-9B (math); Qwen3.5-9B (code) \\
Framework / algorithm & \texttt{verl}; GRPO-style optimization \\
Training data & DAPO Math (math); TACO training split with executable rewards (code) \\
Evaluation data & AIME 2024/2025/2026 and BeyondAIME; TACO test split, LiveCodeBench-v6, HumanEval+, and MBPP+ \\
Evaluation sampling & 16 responses/problem (math); one response/problem per evaluation (code) \\
Prompt batch / rollout group & 128 prompts / 8 responses per prompt \\
PPO mini-batch size & 8 \\
Training budget & Up to 120 training iterations \\
Maximum response length & 16,384 tokens \\
Learning rate & $3\times10^{-6}$ \\
PPO clipping band & $[1-0.2,1+0.28]=[0.8,1.28]$ \\
Compared methods & Null; IcePop; TRM; GeoMean; \methodours \\
Baseline settings & IcePop token-ratio interval $[0.5,5.0]$; TRM cutoff $0.05$ \\
GeoMean thresholds & \makecell[l]{Qwen3.5-4B: $\{1.0025,1.005,1.01,1.02,1.03\}$\\
                                  Qwen3.5-9B: $\{1.0025,1.005,1.01,1.02,1.03\}$} \\
\methodours{} thresholds & \makecell[l]{Qwen3.5-4B: $\{1.02,1.025,1.03,1.04\}$\\
                                      Qwen3.5-9B: $\{1.02,1.025,1.03,1.04\}$} \\
\bottomrule
\end{tabular}
\endgroup
\end{table}

\paragraph{Evaluation protocol.}
Both tasks are evaluated with three random seeds. Mathematical-reasoning evaluation uses 16
responses per problem and reports mean@16; code-generation evaluation uses
one response per problem and reports pass@1. Code decoding uses temperature
$0.7$, top-$p=0.9$, top-$k=50$, and a maximum response length of 16,384 tokens.
Generated programs are judged with the same batched SandboxFusion evaluator
across methods.

For each method, math results use the training iterate with the highest
three-year AIME mean@16; code results use the iterate with the highest
LiveCodeBench-v6 validation pass@1. Best- and worst-of-$k$ use the same math iterate
without separate optimization of either metric.
Qwen3.5-4B BeyondAIME results for Null, GeoMean, and \method{} use separate
runs matched in configuration and training iteration.

\paragraph{TRM implementation.}
We follow the sampled-token max-filtering option described in
Section~4.3 and Appendix~E of \citet{li2026trust}. It rejects a response when
$\max_t \frac12(\log r_t)^2>0.05$.
This is the paper's proposed approximation when full logits are unavailable,
not a full-vocabulary KL implementation. All TRM results in this paper refer
to this sampled-token variant.

\FloatBarrier
\clearpage
\section{Threshold Results}
\label{app:threshold_results}

Table~\ref{tab:threshold_sweep} reports the GeoMean and \method{} threshold
sweeps, with AIME mean@16 averaged equally across the three years.

\begin{table}[H]
\centering
\normalsize
\caption{Threshold sweep (\%, higher is better). Each row reports the
three-year AIME mean@16 for that threshold. Bold and underline denote the
best and second-best values within each model block.}
\label{tab:threshold_sweep}
\setlength{\tabcolsep}{3.6pt}
\begin{tabular}{lcc@{\qquad}lcc}
\toprule
\multicolumn{3}{c}{Qwen3.5-4B} & \multicolumn{3}{c}{Qwen3.5-9B} \\
\cmidrule(lr){1-3}\cmidrule(lr){4-6}
Method & $\tau$ & Avg. & Method & $\tau$ & Avg. \\
\midrule
Null & -- & 50.83 & Null & -- & 70.42 \\
GeoMean & 1.0025 & 69.03 & GeoMean & 1.0025 & 71.67 \\
GeoMean & 1.005 & 69.17 & GeoMean & 1.005 & 74.03 \\
GeoMean & 1.01 & 66.74 & GeoMean & 1.01 & 69.65 \\
GeoMean & 1.02 & 52.57 & GeoMean & 1.02 & 71.67 \\
GeoMean & 1.03 & 52.43 & GeoMean & 1.03 & 69.86 \\
\methodours & 1.02 & \bestresult{70.42} & \methodours & 1.02 & 72.99 \\
\methodours & 1.025 & \secondbest{69.24} & \methodours & 1.025 & \bestresult{77.50} \\
\methodours & 1.03 & 67.99 & \methodours & 1.03 & 73.82 \\
\methodours & 1.04 & 66.81 & \methodours & 1.04 & \secondbest{74.17} \\
\bottomrule
\end{tabular}
\end{table}

\section{Additional Training Diagnostics}
\label{app:additional_validation}

Figure~\ref{fig:masked_fraction} gives the time-resolved utilization view used
in Section~\ref{sec:sequence_utilization}. It distinguishes sustained filtering
from a mean dominated by a small number of unusually strict updates, and
contrasts Qwen3.5-9B on mathematical reasoning and code generation.

\begin{figure}[H]
    \centering
    \includegraphics[width=0.92\linewidth]{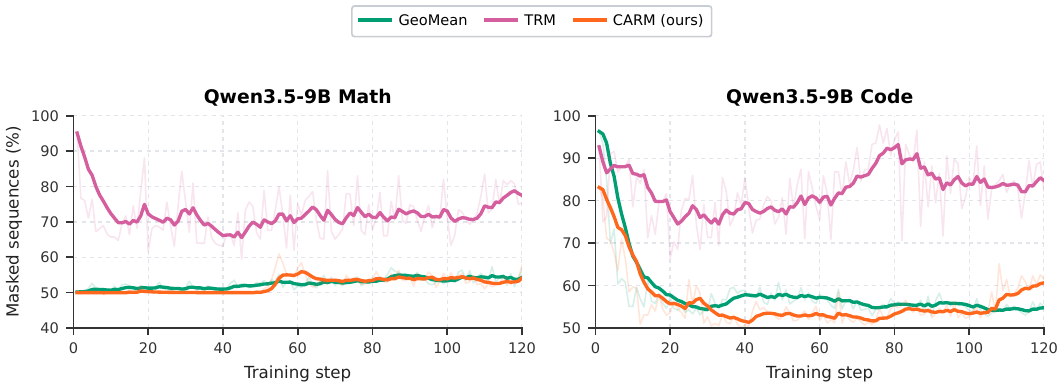}
    \caption{Sequence-masked fraction during training for the selected
    GeoMean, TRM, and \methodours{} configurations. The left and right panels
    show Qwen3.5-9B mathematical-reasoning and code-generation training,
    respectively. Faint curves are raw measurements and opaque curves apply an
    exponential moving average with weight $0.8$.}
    \label{fig:masked_fraction}
\end{figure}

Figure~\ref{fig:training_diagnostics_9b} complements validation accuracy with
generation- and optimization-level measurements. It compares the same five
methods on Qwen3.5-9B mathematical-reasoning training. Faint solid curves show
raw measurements, while opaque solid curves show exponential moving averages
with weight $0.8$. Gradient norm uses a logarithmic axis.

\begin{figure}[H]
    \centering
    \includegraphics[width=\linewidth]{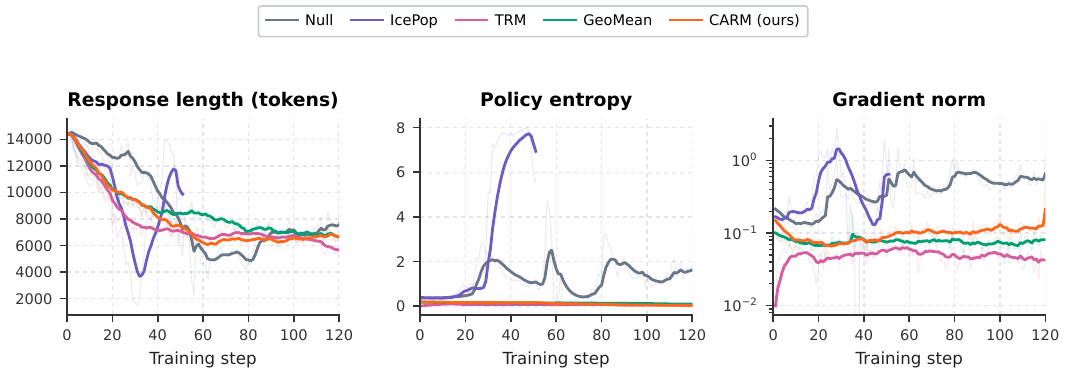}
    \caption{Training diagnostics for Qwen3.5-9B mathematical reasoning.
    Panels show response length, policy entropy, and gradient norm across
    the five methods. Faint curves show raw measurements; opaque curves
    show exponential moving averages with weight $0.8$. Gradient norm is
    displayed on a logarithmic scale.}
    \label{fig:training_diagnostics_9b}
\end{figure}

Figure~\ref{fig:code_training_diagnostics} presents the matched diagnostics
for code-generation training. The same metrics and smoothing convention make
the domain-level differences directly comparable with
Figure~\ref{fig:training_diagnostics_9b}.

\begin{figure}[H]
    \centering
    \includegraphics[width=\linewidth]{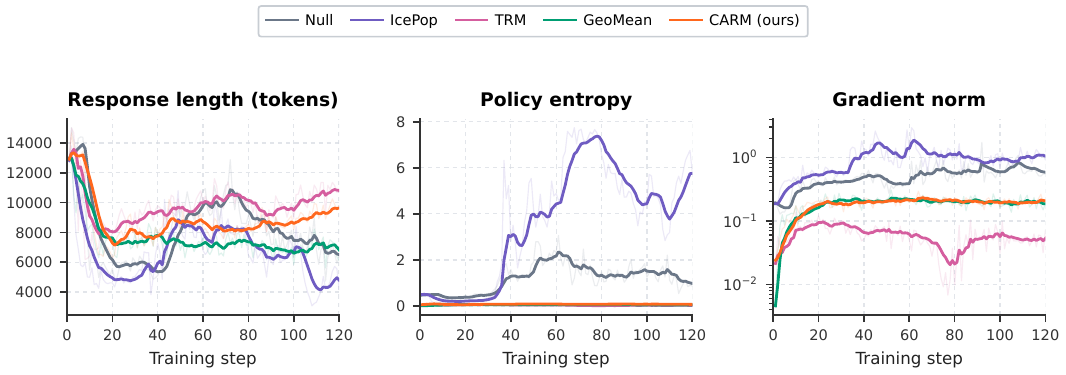}
    \caption{Code-generation training diagnostics: response length, policy
    entropy, and gradient norm for Qwen3.5-9B trained on TACO. Faint and opaque
    curves show raw and smoothed measurements, respectively, using the same
    smoothing and axis conventions as the mathematical-reasoning diagnostics.}
    \label{fig:code_training_diagnostics}
\end{figure}

\FloatBarrier

\section{Best- and Worst-of-\texorpdfstring{$k$}{k} Sampling Metrics}
\label{app:sampling_scaling}

Figure~\ref{fig:sampling_scaling_9b} evaluates inference-time sample scaling
at the same Qwen3.5-9B checkpoints reported in
Table~\ref{tab:sampling_scaling_values}. For each problem, let
$X_1,\ldots,X_{16}\in\{0,1\}$ denote the correctness of the 16 independently
sampled responses and let $c=\sum_{j=1}^{16}X_j$. The per-problem empirical
accuracy is $c/16$, and \texttt{mean@16} averages this quantity over problems
and then equally over AIME 2024, 2025, and 2026.

Best-of-$k$ asks whether \emph{at least one} of $k$ sampled responses is
correct. It captures the opportunity provided by additional inference-time
sampling when an oracle verifier can select a correct candidate. Worst-of-$k$
asks whether \emph{all} $k$ sampled responses are correct. It is therefore a
strict measure of repeatability: a model scores highly only when its stochastic
samples solve the same problem consistently. Worst-of-$k$ should not be read as
robustness to distribution shift or adversarial perturbations. At $k=1$, both
statistics reduce exactly to \texttt{mean@16}, which anchors the two curves at
the same observed single-sample accuracy.

For $k\in\{2,4,8,16\}$, we resample $k$ responses with replacement from the 16
available responses for each problem and apply the maximum (best-of-$k$) or
minimum (worst-of-$k$) correctness operator. The plotted values average 1,000
bootstrap replicates, using the same models as Table~\ref{tab:main_results}.

For a problem with $c$ correct responses among 16, the expected best-of-$k$
value is
\begin{equation}
\widehat{B}_k(c)
=1-\left(1-\frac{c}{16}\right)^k.
\label{eq:bootstrap_best_of_k}
\end{equation}
Equation~\ref{eq:bootstrap_best_of_k} is the probability that at least one of
$k$ independent draws from the empirical response distribution is correct.
It is a with-replacement best-of-$k$ statistic, not the standard
without-replacement pass@$k$ estimator. The corresponding expected
worst-of-$k$ value is
\begin{equation}
\widehat{W}_k(c)
=\left(\frac{c}{16}\right)^k.
\label{eq:bootstrap_worst_of_k}
\end{equation}
Equation~\ref{eq:bootstrap_worst_of_k} is the probability that all $k$ draws
are correct. It is non-increasing in $k$, and strictly decreasing for a
problem when $0<c<16$. As with best-of-$k$, this is a with-replacement statistic
rather than the literal maximum or minimum over a fixed subset of the original
16 responses. Together, the two curves characterize solution coverage and
repeatability. Worst-of-$k$ emphasizes problems with high per-problem
correctness probability; an aggregate increase does not imply improvement
on every problem.

\begin{table}[H]
\centering
\normalsize
\caption{Exact Qwen3.5-9B values underlying
Figure~\ref{fig:sampling_scaling_9b}. All entries are equal-weighted
three-year AIME averages in percent.}
\label{tab:sampling_scaling_values}
\begin{tabular}{llccccc}
\toprule
Metric & Method & $k=1$ & $k=2$ & $k=4$ & $k=8$ & $k=16$ \\
\midrule
\multirow{4}{*}{Best-of-$k$}
& Null & 70.42 & 78.07 & 83.34 & 86.42 & 88.26 \\
& TRM & 73.96 & \secondbest{81.47} & \secondbest{86.43} & \secondbest{89.39} & \secondbest{90.98} \\
& GeoMean & \secondbest{74.03} & 80.19 & 84.05 & 86.73 & 88.54 \\
& \methodours & \bestresult{77.50} & \bestresult{83.94} & \bestresult{87.74} & \bestresult{90.05} & \bestresult{91.28} \\
\midrule
\multirow{4}{*}{Worst-of-$k$}
& Null & 70.42 & 62.70 & 55.30 & 49.04 & 44.38 \\
& TRM & 73.96 & 66.49 & 59.27 & 52.99 & 48.00 \\
& GeoMean & \secondbest{74.03} & \secondbest{67.66} & \secondbest{61.32} & \secondbest{55.02} & \secondbest{49.83} \\
& \methodours & \bestresult{77.50} & \bestresult{71.04} & \bestresult{64.44} & \bestresult{58.05} & \bestresult{52.61} \\
\bottomrule
\end{tabular}
\end{table}

\FloatBarrier

\end{document}